\documentclass{article}

\newif\ifisarxiv
\isarxivtrue

\ifisarxiv
\usepackage[final,nonatbib]{arxiv_2017}
\else
\usepackage[nonatbib]{sty/nips_2017}
\fi
\PassOptionsToPackage{numbers, compress}{natbib}

\usepackage[utf8]{inputenc} 
\usepackage[T1]{fontenc}    
\usepackage{hyperref}       
\usepackage{url}            
\usepackage{booktabs}       
\usepackage{amsfonts}       
\usepackage{nicefrac}       
\usepackage{microtype}      
\usepackage{cancel}      
\usepackage[font=small,labelfont=bf]{caption}

\usepackage{times}
\usepackage{graphicx} 
\usepackage{subfigure} 
\usepackage{color}

\usepackage{amsmath}
\usepackage{verbatim}

\usepackage{algorithm}
\usepackage{algorithmic}
\usepackage{enumitem}

\usepackage{tikz,pgfplots}
\pgfplotsset{compat=newest}

\usepackage{forloop}
\newcounter{loopcntr}

\usepackage{wrapfig}
\usepackage{eucal}
\newcommand{\ofsub}[1]{\mbox{\small \raisebox{0.0pt}{$(#1)$}}}
\newcommand{\of}[2]{{#1{\!\ofsub{#2}}}}
\newcommand{\fof}[2]{{#1({#2})}}
\newcommand{\yof}[2]{{#1{\ofsub{#2}}}}

\newcommand{\Sm}{{S_{-i}}}
\newcommand{\Sp}{{S_{+i}}}
\ifx\BlackBox\undefined
\newcommand{\BlackBox}{\rule{1.5ex}{1.5ex}}  
\fi

\DeclareMathOperator*{\argmin}{\mathop{\mathrm{argmin}}}
\def\x{\mathbf x}
\def\y{\mathbf y}
\def\ybh{\widehat{\mathbf y}}
\def\yh{\widehat{y}}

\def\w{\mathbf w}
\def\v{\mathbf v}

\def\e{\mathbf e}

\def\X{\mathbf X}
\def\Xs{\widetilde{\X}}

\def\C{\mathbf C}

\def\F{\mathbf F}

\def\Z{\mathbf Z}
\def\I{\mathbf I}

\def\E{\mathbb E}
\def\R{\mathbb R} 
\def\tr{\mathrm{tr}}

\def\rank{\mathrm{rank}}

\def\Var{\mathrm{Var}}

\def\XinvS{(\X_S\X_S^\top)^{-1}}

\newcommand{\defeq}{\stackrel{\text{\tiny{def}}}{=}}

\definecolor{silver}{cmyk}{0,0,0,0.3}
\definecolor{yellow}{cmyk}{0,0,0.9,0.0}
\definecolor{reddishyellow}{cmyk}{0,0.22,1.0,0.0}
\definecolor{black}{cmyk}{0,0,0.0,1.0}
\definecolor{darkYellow}{cmyk}{0.2,0.4,1.0,0}
\definecolor{darkSilver}{cmyk}{0,0,0,0.1}
\definecolor{grey}{cmyk}{0,0,0,0.5}

\newcommand{\Blue}[1]{\color{blue}{#1}\color{black}}
\newcommand{\Brown}[1]{{\color{brown}{#1}\color{black}}}

\ifx\proof\undefined
\newenvironment{proof}{\par\noindent{\bf Proof\ }}{\hfill\BlackBox\\[2mm]}
\fi

\ifx\theorem\undefined
\newtheorem{theorem}{Theorem}
\fi

\ifx\example\undefined
\newtheorem{example}{Example}
\fi

\ifx\property\undefined
\newtheorem{property}{Property}
\fi

\ifx\lemma\undefined
\newtheorem{lemma}[theorem]{Lemma}
\fi

\ifx\proposition\undefined
\newtheorem{proposition}[theorem]{Proposition}
\fi

\ifx\remark\undefined
\newtheorem{remark}[theorem]{Remark}
\fi

\ifx\corollary\undefined
\newtheorem{corollary}[theorem]{Corollary}
\fi

\ifx\definition\undefined
\newtheorem{definition}{Definition}
\fi

\ifx\conjecture\undefined
\newtheorem{conjecture}[theorem]{Conjecture}
\fi

\ifx\axiom\undefined
\newtheorem{axiom}[theorem]{Axiom}
\fi

\ifx\claim\undefined
\newtheorem{claim}[theorem]{Claim}
\fi

\ifx\assumption\undefined
\newtheorem{assumption}[theorem]{Assumption}
\fi

\title{Unbiased estimates for linear regression\\
 via volume sampling}

\author{
Micha{\l } Derezi\'{n}ski\\
Department of Computer Science\\
University of California Santa Cruz\\
\texttt{mderezin@ucsc.edu}\\
\And
Manfred K. Warmuth\\
Department of Computer Science\\
University of California Santa Cruz\\
\texttt{manfred@ucsc.edu}
}

\begin{document}

\maketitle

\begin{abstract}
Given a full rank matrix $\X$ with more columns than rows,
consider the task of estimating the pseudo inverse $\X^+$ based 
on the pseudo inverse of a sampled subset of columns
(of size at least the number of rows). We show that this is possible
if the subset of columns is chosen proportional to the squared volume
spanned by the rows of the chosen submatrix (ie, volume sampling).
The resulting estimator is unbiased and
surprisingly the covariance of the estimator also has a
closed form: It equals a specific factor times
$\X^+\X^{+\top}$.

Pseudo inverse plays an important part in solving the linear least
squares problem, where we try to predict a label for each column of
$\X$. We assume labels are expensive and we are
only given the labels for the small subset of columns we sample
from $\X$. Using our methods we show that
the weight vector of the solution for the sub problem
is an unbiased estimator of the
optimal solution for the whole problem based on all 
column labels. 

We believe that these new formulas establish a fundamental
connection between linear least squares and volume sampling.
We use our methods to obtain an algorithm for volume
sampling that is faster than state-of-the-art and for obtaining
bounds for the total loss of the estimated least-squares solution on
all labeled columns. 
\end{abstract}



\section{Introduction}
\label{sec:introduction}

\begin{wrapfigure}{r}{0.2\textwidth}
\vspace{-.7cm}
\mbox{\footnotesize $\X$}\hfill
\begin{tikzpicture}[scale=0.4,baseline=(current bounding box.center)]
    \draw [fill=brown!30] (0,0) rectangle (4,2);
    \draw [color=black] (2.5,0) -- (2.5,2);
    \draw (2,1) node {\mbox{\footnotesize $\x_i$}}; 
\end{tikzpicture} 
\vspace{.05cm}

\mbox{\footnotesize $\I_S$}\hfill
\begin{tikzpicture}[scale=0.4,baseline=(current bounding box.center)]
    \draw (0,0) rectangle (4,3.9);
    \draw [color=blue] (.45,3.5) -- (2.7,1.3);
    \draw (1.75,3) node {\mbox{\footnotesize $\Blue{S}$}}; 
\end{tikzpicture} 
\vspace{.05cm}

\mbox{\footnotesize $\X\I_S$} \hfill
\begin{tikzpicture}[scale=0.4,baseline=(current bounding box.center)]
    \draw (0,0) rectangle (4,2);
    \draw[fill=blue!30] (.5,0) rectangle (2.7,2);
    \draw (1.55,.5) node {\mbox{\footnotesize $\Blue{\X_S}$}};
\end{tikzpicture} 
\vspace{.03cm}

\mbox{\footnotesize $\X^{+\top}$}\hfill
\begin{tikzpicture}[scale=0.4,baseline=(current bounding box.center)]
    \draw [fill=brown!30] (0,0) rectangle (4,2);
\end{tikzpicture} 
\vspace{.05cm}

\mbox{\footnotesize$\!(\!\X\I_S\!)^{\!+\!\top}\!$}\hfill
\begin{tikzpicture}[scale=0.4,baseline=(current bounding box.center)]
    \draw (0,0) rectangle (4,2);
    \draw[fill=blue!30] (.5,0) rectangle (2.7,2);
    \draw (1.55,.5) node {\mbox{\footnotesize $\;\Blue{(\!\X_S\!)^{\!+\!\top}}$}};
\end{tikzpicture} 
\vspace{-.1cm}
  \caption{Set $S$ may not be consecutive.}
  \label{f:shapes}
\vspace{-1.3cm}
\end{wrapfigure}
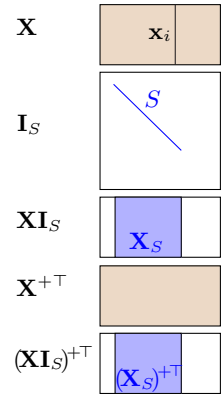
Let $\X$ be a wide full rank matrix 
with $d$ rows and $n$ columns where $n \ge d$.
Our goal is to estimate the pseudo inverse $\X^+$ of $\X$ 
based on the pseudo inverse of a subset of columns. More precisely, 
we sample a subset $S\subseteq \{1..n\}$ of $s$ column indices
(where $s\geq d$). We let $\X_S$ be the sub-matrix of the $s$
columns indexed by $S$ (See Figure \ref{f:shapes}).
Consider a version of $\X$ in which
all but the columns of $S$ are zero. 
This matrix equals $\X\I_S$ where $\I_S$ is an $n$-dimensional diagonal 
matrix with $(\I_S)_{ii}=1$ if $i\in S$ and 0 otherwise.

We assume that the set of $s$ column indices of $\X$ is selected 
proportional to the squared volume spanned by the rows of submatrix
$\X_S$, i.e.  proportional to $\det(\X_S\X_S^\top)$ 
and prove a number of new surprising expectation formulas for this type of volume
sampling, such as
$$
\E[(\X\I_S)^+]=\X^+ \quad \text{and}\quad
\E[\underbrace{(\X_S\X_S^\top)^{-1}}_
{(\X\I_S)^{+\top}(\X\I_S)^+} ]= \frac{n-d+1}{s-d+1}\,
\X^{+\top}\X^+.$$
Note that $(\X\I_S)^+$ has the $n\times d$ shape of $\X^+$
where the $s$ rows indexed by $S$ contain $(\X_S)^+$ and the
remaining $n-s$ rows are zero.
The expectation of this matrix is $\X^+$ even though
$(\X_S)^+$ is clearly not a sub-matrix of $\X^+$.
In addition to the
expectation formulas, our new techniques lead to
an efficient volume sampling procedure which beats the state-of-the-art
by a factor of $n^2$ in time complexity.

Volume sampling is useful in numerous applications, from clustering to
matrix approximation, but we focus on the task of solving
linear least squares problems: For an $n-$dimensional
label vector $\y$, let $\w^*=\argmin_\w ||\X^\top\w-\y||^2=\X^+\y$.
Assume the entire design matrix $\X$ is
known to the learner but labels are expensive 
and you want to observe as few of them as possible. 
Let $\of{\w^*}S=(\X_S)^+\y_S$ be the solution to the sub-problem based 
on labels $\y_S$. 
What is the smallest number of labels $s$ necessary, for
which there is a sampling procedure on sets $S$ of size
$s$ st the expected loss of $\of{\w^*}S$ is at most a
constant factor larger than the loss of $\w^*$ that uses all $n$
labels (where the constant is independent of $n$)?
More precisely, using the short hand $L(\w)=||\X^\top\w-\y||^2$
for the loss on all $n$ labels, what is the smallest size $s$ such that
$\E[L(\of{\w^*}S)]\le \text{const}\, L(\w^*)$. This question is a version of
the ``minimal coresets'' open problem posed in \cite{coresets-regression}.

The size has to be at least $d$ and
one can show that randomization is necessary
in that any deterministic algorithm for choosing a set of
$d$ columns can suffer loss larger by a factor of $n$.
Also any iid sampling of $S$ (such as the commonly used leverage scores
\cite{fast-leverage-scores})
requires at least $\Omega(d \log d)$ examples to achieve a finite factor.
In this paper however we show that with a size $d$ volume
sample,
$\E[L(\of{\w^*}S)]=(d+1)L(\w^*)$ if $\X$ is in general position.
Note again that we have equality and not just an upper
bound. Also we can show that the multiplicative factor $d+1$ is optimal.
We further improve this factor to
$1+\epsilon$ via repeated volume sampling. Moreover, our
expectation formulas imply that when $S$ is size $s\ge d$ volume 
sampled, then $\of{\w^*}S$ is an unbiased estimator for
$\w^*$, ie $\E[\of{\w^*}S]=\w^*$.

\vspace{-.1cm}
\section{Related Work}
\label{sec:related-work}
Volume sampling is an extension of a determinantal point
process \cite{dpp}, which has been given a lot of attention in the literature
with many applications to machine learning, including
recommendation systems \cite{dpp-shopping} and clustering
\cite{dpp-clustering}. Many exact and approximate methods for efficiently
generating samples from this distribution have been proposed
\cite{efficient-volume-sampling,k-dpp}, making it a useful tool in
the design of randomized algorithms. Most of those methods focus on
sampling $s\leq d$ elements. In this paper, we study volume sampling
sets of size $s\geq d$, which has been proposed in
\cite{avron-boutsidis13} and motivated with applications in graph
theory, linear regression, matrix approximation and more. The only
known polynomial time algorithm for size $s>d$ volume sampling was
recently proposed in \cite{dual-volume-sampling} with time complexity
$O(n^4 s)$. We offer a new algorithm with runtime $O((n-s+d)nd)$,
which is faster by a factor of at least $n^2$.  

The problem of selecting a subset of input vectors for solving a linear
regression task has been extensively studied in statistics literature
under the terms {\em optimal design} \cite{optimal-design-book} and
{\em pool-based active learning}
\cite{pool-based-active-learning-regression}. Various 
criteria for subset selection have been proposed, like A-optimality
and D-optimality. For example, A-optimality seeks to minimize
$\tr((\X_S\X_S^\top)^{-1})$, which is combinatorially hard to optimize
exactly. 
We show that for size $s$ volume sampling (for $s\geq d$), 
$\E[(\X_S\X_S^\top)^{-1}] = \frac{n-d+1}{s-d+1}\,
\X^{+\top}\X^+$ which provides an approximate
randomized solution for this task.

A related task has been explored in the field of computational
geometry, where efficient algorithms are sought for approximately
solving linear regression and matrix approximation
\cite{randomized-matrix-algorithms,
  regression-input-sparsity-time,coresets-regression}. Here, 
subsampling appears as one of the key 
techniques for obtaining multiplicative bounds on the loss of the
approximate solution. In this line of work, volume sampling
size $s\leq d$ has been used by
\cite{pca-volume-sampling,more-efficient-volume-sampling} for 
matrix approximation. Another common sampling technique is based on
statistical leverage scores \cite{fast-leverage-scores}, which have
been effectively used for the task of linear regression. However, this
approach is based on iid sampling, and requires 
sampling at least $\Omega(d\log d)$ elements to achieve multiplicative
loss bounds. On the other hand, the input vectors obtained from volume
sampling are selected jointly, which makes the chosen subset more
informative, and we show that just $d$ volume sampled elements are
sufficient to achieve a multiplicative bound. 

\ifisarxiv
\vspace{-2mm}
\fi
\section{Unbiased estimators}
\label{sec:pseudo-inverse}
\ifisarxiv
\vspace{-1mm}
\fi
Let $n$ be an integer dimension.
For each subset $S\subseteq \{1..n\}$ of size $s$ we are given a 
matrix formula $\fof{\F}S$.
Our goal is to sample set $S$ of size $s$ using some
sampling process and then develop concise expressions for
$\E_{S:|S|=s}[\fof{\F}S]$. Examples of formula
classes $\fof{\F}S$ will be given below.

We represent the sampling by a directed acyclic graph (dag), with a
single root node corresponding to the full set $\{1..n\}$, Starting
from the root, we proceed along the edges of the graph, 
iteratively removing elements from the set $S$. 
Concretely, consider a dag with levels $s = n, n-1, ..., d$.
Level $s$ contains $n \choose s$ nodes 
for sets $S\subseteq \{1..n\}$ of size $s$. 
Every node $S$ at level $s>d$ has $s$ directed edges
to the nodes $S-\{i\}$ at the next lower level. 
These edges are labeled with a
conditional probability vector $P(S_{-i}|S)$.
The probability of a (directed) path is the product of the
probabilities along its edges.
The outflow of probability from each node on
all but the bottom level is 1. 
We let the probability
$P(S)$ of node $S$ be the probability of all paths 
from the top node $\{1..n\}$ to $S$ and set the
probability $P(\{1..n\})$ of the top node to 1.
We associate a formula $\fof{\F}S$ with each set node $S$ in the
dag. The following key equality lets us compute
expectations.

\begin{lemma}
\label{l:key}
If for all $S\subseteq \{1..n\}$ of size greater than $d$ we have
$$\Blue{\fof{\F}{S}=\sum_{i\in S} P(\Sm|S)\fof{\F}{S_{-i}}},$$
then for any $s\in\{d..n\}$: $\;\;\E_{S:|S|=s} [\fof{\F}{S}]
=\sum_{S:|S|=s} P(S)\fof{\F}{S} = \fof{\F}{\{1..n\}}.$
\end{lemma}
\noindent{\bf Proof}$\;$
Suffices to show that expectations at successive layers are equal:
$$
\sum_{S:|S|=s} \!\!P(S) \,\Blue{\fof{\F}{S}}
= \!\!\sum_{S:|S|=s} \!\!P(S) \Blue{\sum_{i\in S} P(S_{-i}|S) \,\fof{\F}{S_{-i}}}
= \!\!\!\!\sum_{T:|T|=s-1} \underbrace{\sum_{j\notin T} P(T_{+j}) P(T|T_{+j})}_{P(T)} \fof{\F}{T}.
\ \BlackBox
$$
\ifisarxiv
\vspace{-1cm}
\else
\vspace{-.5cm}
\fi
\subsection{Volume sampling}
Given a wide full-rank matrix $\X\in\R^{d\times n}$
and a sample size $s\in \{d..n\}$,
volume sampling chooses subset $S\subseteq\{1..n\}$
of size $s$ with probability proportional to volume spanned by
the rows of submatrix $\X_S$, ie proportional to
$\det(\X_S\X_S^\top)$. The following corollary uses
the above dag setup to compute the normalization constant for this
distribution. When $s=d$, the corollary provides a novel minimalist
proof for the Cauchy-Binet formula: $\sum_{S:|S|=s}\det(\X_S\X_S^\top)=\det(\X\X^\top)$.
\begin{corollary}
\label{c:vol}
Let $\X\in \R^{d\times n}$
and $S\in \{1..n\}$ of size $n\ge s \ge d$ st
$\det(\X_S\X_S^\top)>0$. Then for any $i\in S$, define
\begin{align*}
\!\!
P(\Sm|S)\!:=\!\frac{\det(\X_\Sm\X_\Sm^\top)}
             {(s\!-\!d)\det(\X_S\X_S^\top)}
\!=\! \frac{1\!-\!\x_i^\top (\X_S\X_S^\top)^{-1}\x_i} {s\!-\!d},
\tag{\bf reverse iterative volume sampling}
\end{align*}
where $\x_i$ is the $i$th column of $\X$
and $\X_S$ is the sub matrix of columns indexed by $S$.
Then $P(\Sm|S)$ is a proper probability distribution
and thus $\sum_{S:|S|=s} P(S)=1$ for all $s\in\{d..n\}$.
Furthermore
\begin{align*}
P(S)= \frac{\det(\X_S\X_S^\top)} {{n-d \choose s-d}\det(\X\X^\top)}.
\tag{\bf volume sampling}
\end{align*}
\end{corollary}
\noindent
{\bf Proof}$\;$
For any $S$, st $\det(\X_S\X_S^\top)>0$,
it is easy to see that $P(\Sm|S)$ forms a probability
vector:
$$\sum_{i\in S} P(\Sm|S)=\sum_{i\in S}
\frac{1-\tr((\X_S\X_S^\top)^{-1}\x_i\x_i^\top)}{s-d}
= \frac{s-\tr((\X_S\X_S^\top)^{-1}\X_S\X_S^\top)}{s-d}
=\frac{s-d}{s-d}=1.$$

It remains to show the formula for the probability $P(S)$ of all paths
ending at $\fof{\F}S$.
We first consider the top node, ie $\{1..n\}$. In this
case both the path definition and the formula are 1.
For all but the top node,
the probability $P(S)$ equals the 
total inflow of probability into that node from the
previous level, ie
\begin{align*}
P(S)=\sum_{i\notin S} P(S|\Sp) \; P(\Sp)
&=\sum_{i\notin S}
\frac{\det(\X_S\X_S^\top)} 
     {(s+1-d)\cancel{\det(\X_\Sp\X_\Sp^\top)}}
\frac{\cancel{\det(\X_\Sp\X_\Sp^\top)}}
     {{n-d\choose s+1-d}\det(\X\X^\top)}
\\&=
\frac{(n-s) \det(\X_S\X_S^\top)} {(s+1-d) {n-d \choose
s+1-d} \det(\X\X^\top)}
= \frac{\det(\X_S\X_S^\top)} {{n-d \choose s-d} \det(\X\X^\top)}. 
\quad\quad\BlackBox
\end{align*}
Note that all paths from $S$ to a subset $T$ (of size $\ge d$)
have the same probability because the ratios of
determinants cancel along paths.
  It is easy to verify
 that this probability is 
 $\frac{\det(\X_T\X_T^\top)} 
       {(|S|-|T|)!\, {|S|-d \choose |T|-d}
 \det(\X_S\X_S^\top)}.$
Note that $\frac{0}{0}$ ratios are avoided because paths
with such ratios always lead to sets of probability
0.

\subsection{Expectation formulas for volume sampling}
All expectations in the remainder of the paper are wrt
volume sampling. We use the short hand $\E[\fof{\F}S]$ for
expectation with volume sampling where the size of the
sampled set is fixed to $s$.
The expectation formulas for two choices of $\fof{\F}S$ are
proven in the next two theorems. By Lemma \ref{l:key} 
it suffices to show $\fof{\F}S=\sum_{i\in S} P(\Sm|S)\fof{\F}\Sm$
for volume sampling.

We introduce a bit more notation first. 
Recall that $\X_S$ is the sub matrix of
columns indexed by $S\subseteq\{1..n\}$ 
(See Figure \ref{f:shapes}).
Consider a version of $\X$ in which
all but the columns of $S$ are zero.
This matrix equals $\X\I_S$ where $\I_S$ is an
$n$-dimensional diagonal
matrix with $(\I_S)_{ii}=1$ if $i\in S$ and 0 otherwise.

\begin{theorem}\label{t:einv}
Let $\X\in\R^{d\times n}$ be a wide full rank matrix
(ie $n\geq d$). For $s\in \{d..n\}$, let
$S\subseteq 1..n$ be a size $s$ volume sampled set over $\X$. Then
$$\E[(\X\I_S)^+]=\X^+.$$
\end{theorem}
We believe that this fundamental formula lies at the core of why
volume sampling is important in many applications. In this work, we
focus on its application to linear regression. However,
\cite{avron-boutsidis13} discuss many problems where controlling the
pseudo-inverse of a submatrix is essential. For those
applications, it is important to establish variance bounds for the
estimator offered by Theorem \ref{t:einv}. In this case, volume
sampling once again offers very concrete guarantees. We obtain them by
showing the following formula, which can be viewed as a second moment
for this estimator.
\begin{theorem}\label{t:einvs}
Let $\X\in\R^{d\times n}$ be a full-rank matrix and $s\in\{d..n\}$.
If size $s$ volume sampling over $\X$ has full support, then
\[\E[\underbrace{(\X_S\X_S^\top)^{-1}}_
{(\X\I_S)^{+\top}(\X\I_S)^+} ]
= \frac{n-d+1}{s-d+1}\,
\underbrace{(\X\X^\top)^{-1}}_{\X^{+\top}\X^+}.\]
If volume sampling does not have full support then
the matrix equality ``$=$'' is replaced by the positive-definite
inequality ``$\preceq$''.
\end{theorem}
The condition that size $s$ volume sampling over $\X$ has full support is
equivalent to $\det(\X_S\X_S^\top)>0$ for
all $S\subseteq 1..n$ of size $s$.
Note that if size $s$ volume sampling has full support, then size
$t>s$ also has full support. So full support for the
smallest size $d$
(often phrased as $\X$ being {\em in general position})
implies that volume sampling wrt any size $s\ge d$ has full support.

Surprisingly by combining theorems \ref{t:einv}
and \ref{t:einvs}, we can 
obtain a ``covariance type formula'' for the pseudo-inverse 
matrix estimator:
\begin{align}
&\E[((\X\I_S)^+-\E[(\X\I_S)^+])^\top\; ((\X\I_S)^+-\E[(\X\I_S)^+])]  
\nonumber\\ 
&=\E[(\X\I_S)^{+\top}(\X\I_S)^{+}]  - \E[(\X\I_S)^{+}]^\top\; \E[(\X\I_S)^+]
\nonumber\\ 
&=\frac{n-d+1}{s-d+1} \;\X^{+\top}\X^{+} - \X^{+\top}\X^{+}
  =\frac{n-s}{s-d+1}\; \X^{+\top}\X^+.
\label{e:covs}
\end{align}
Theorem \ref{t:einvs} can also be used to obtain an expectation formula for
the Frobenius norm
$\|(\X\I_S)^+\|_F$ of the estimator:
\begin{align}
\label{e:frobs}
\E\|(\X\I_S)^+\|_F^2 &= \E[\tr((\X\I_S)^{+\top}
  (\X\I_S)^+)] = \frac{n-d+1}{s-d+1}\|\X^+\|_F^2.
\end{align}
This norm formula has been shown in \cite{avron-boutsidis13}, with
numerous applications. Theorem \ref{t:einvs} can be viewed as a
much stronger pre trace version of the norm formula.
Also our proof techniques are quite different and much simpler.
Note that if size $s$ volume sampling for $\X$ does not have
full support then \eqref{e:covs} becomes
a semi-definite inequality $\preceq$ between matrices and
\eqref{e:frobs} an inequality between numbers.

{\bf Proof of Theorem \ref{t:einv}}$\,$
We apply Lemma \ref{l:key} with $\fof{\F}S= (\X\I_S)^+$.
It suffices to show $\fof{\F}S=\sum_{i\in S} P(\Sm|S)\fof{\F}\Sm$ 
for $P(\Sm|S):=\frac{1-\x_i^\top(\X_S\X_S^\top)^{-1}\x_i}{s-d}$, ie:
$$(\X\I_S)^+
= \sum_{i\in S} \frac{1-\x_i^\top\XinvS\x_i}{s-d}
\underbrace{(\X\I_\Sm)^+}
_{(\X\I_\Sm)^\top (\X_\Sm \X_\Sm^\top)^{-1} }.$$
Proven by applying Sherman Morrison to
$(\X_\Sm \X_\Sm^\top)^{-1}=
(\X_S\X_S^\top-\x_i\x_i^\top)^{-1}$ on the rhs:
$$\sum_i \frac{1-\x_i^\top\XinvS\x_i}  
               {n-d} \quad
((\X\I_S)^\top-\e_i\x_i^\top)
\left(\XinvS +\frac{\XinvS \x_i\x_i^\top\XinvS} {1-\x_i^\top\XinvS\x_i}
	      \right)             
.$$
We now expand the last two factors into 4 terms.
The expectation of the first
$(\X\I_S)^\top(\X_S\X_S^\top)^{-1}$ is $(\X\I_S)^+$
(which is the lhs) and the expectations of the remaining three terms times $n-d$ 
sum to 0:
\begin{align*}
&-\sum_{i\in S} (1-\x_i^\top\XinvS\x_i)\, \e_i\x_i^\top\XinvS
+(\X\I_S)^\top\cancel{\XinvS} \cancel{\sum_{i\in S} \x_i\x_i^\top} \XinvS 
\\&\qquad 
-\sum_{i\in S}\e_i(\x_i^\top\XinvS \x_i)\;\x_i^\top\XinvS
= 0.
\hspace{5cm} \BlackBox
\end{align*}
{\bf Proof of Theorem \ref{t:einvs}}$\,$
Choose $\fof{\F}S= \frac{s-d+1}{n-d+1} (\X_S\X_S^\top)^{-1}$.
By Lemma \ref{l:key} it suffices to 
show $\fof{\F}S=\sum_{i\in S} P(\Sm|S)\fof{\F}\Sm$ for volume sampling:
$$\frac{s-d+1}{\cancel{n-d+1}} (\X_S\X_S^\top)^{-1}
= \sum_{i\in S} \frac{1-\x_i^\top(\X_S\X_S^\top)^{-1}\x_i}{\cancel{s-d}}
  \frac{\cancel{s-d}}{\cancel{n-d+1}} (\X_\Sm\X_\Sm^\top)^{-1}
$$
To show this we apply Sherman Morrison to 
$(\X_\Sm \X_\Sm^\top)^{-1}$ on the rhs: 
\begin{align*}
&\sum_{i\in S} (1-\x_i^\top\XinvS\x_i)
\left(\XinvS +\frac{\XinvS \x_i\x_i^\top\XinvS}
{1-\x_i^\top\XinvS\x_i}\right)
\\&\Blue{=} \,(s-d) \XinvS
       +  \cancel{\XinvS} \cancel{\sum_{i\in S}
\x_i\x_i^\top} \XinvS
=(s-d+1)\;\XinvS. 
\end{align*}
If some denominators $1-\x_i^\top\XinvS\x_i$ are zero, then
only sum over $i$ for which the denominators are
positive. In this case the above matrix equality becomes a positive-definite inequality $\Blue{\preceq}$.
\hfill\BlackBox


\section{Linear regression with few labels}
\label{sec:linear-regression}
\begin{wrapfigure}{r}{0.4\textwidth}
\vspace{-18mm}
\begin{tikzpicture}[font=\normalsize,scale=0.8,pin distance=1.6mm]
    \begin{axis}[hide axis, xmin=-2.35,xmax=1.35,ymin=-2.2,ymax = 5.1]
        \addplot [domain=-2.07:1.1,samples=250, ultra thick, blue] 
	{x^2} node [pos=0.3, xshift=-.5cm] {$L(\cdot)$};
        \addplot [domain=-2.1:1.1,samples=250, ultra thick, red ] {2-x};
	\addplot[mark=none, ultra thick, green] coordinates {(-2.15,-1) (1.15,-1)};
        \addplot[mark=*] coordinates {(0,0)} node[pin=-20:{$L(\w^*)$}]{};
        \addplot[mark=*] coordinates {(0,2)} node[pin=90:{$\E[L(\of{\w^*}S)]$}]{};
        \addplot[mark=*] coordinates {(-2,4)} node[pin=90:{$\,\,\,L(\of{\w^*}{S_i})$}]{};
        \addplot[mark=*] coordinates {(1,1)} node[pin=90:{$L(\of{\w^*}{S_j})\,\,\,$}]{};
        \addplot[mark=*] coordinates {(-2,-1)} node[pin=-90:{$\of{\w^*}{S_i}$}]{};
        \addplot[mark=*] coordinates {(1,-1)} node[pin=-90:{$\of{\w^*}{S_j}$}]{};
        \addplot[mark=*] coordinates {(0,-1)} node[pin=-90:{$\w^*=\E(\of{\w^*}S)$}]{};
	\draw [decorate,decoration={brace,amplitude=4.5pt},xshift=-2.5pt,yshift=0pt]
	(0,0) -- (0,2) node [black,midway,xshift=-.8cm] {$d\,L(\w^*\!)$};
    \end{axis}
    \end{tikzpicture}
\caption{Unbiased estimator $\of{\w^*}S$ in expectation suffers loss
  $(d+1)\,L(\w^*)$.}
\vspace{0mm}
\end{wrapfigure}
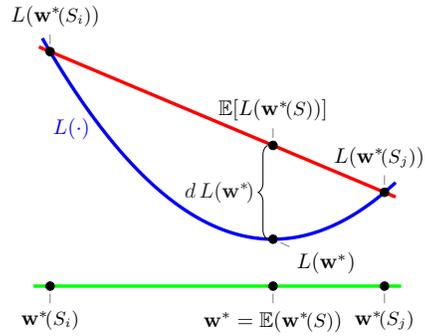
Our main motivation for studying volume sampling came from asking
the following simple question. Suppose we want
to solve a $d$-dimensional linear regression problem with
a matrix $\X\in\R^{d\times n}$ of input column vectors and a label
vector $\y\in\R^n$, ie find
$\w\in\R^d$ that minimizes the least squares loss $L(\w)=\|\X^\top\w-\y\|^2$:
\[\w^*\defeq \argmin_{\w\in\R^d}L(\w)
       =\X^{+\top}\y,\]
but the access to label vector $\y$ is restricted. We
are allowed to pick a subset 
$S\subseteq\{1..n\}$ for which the labels $y_i$ (where $i\in S$) are
revealed to us, and then solve the subproblem $(\X_S,\y_S)$, obtaining
$\of{\w^*}S$. What is the smallest number of labels such that for any
$\X$, we can find $\of{\w^*}S$ for which $L(\of{\w^*}S)$ is only a multiplicative
factor away from $L(\w^*)$ (independent of the number of input vectors
$n$)? This question was posed as an open problem by
\cite{coresets-regression}. It is easy to show that we need at least
$d$ labels (when $\X$ is full-rank), so as to guarantee the 
uniqueness of solution $\of{\w^*}S$. 
We use volume sampling to show that $d$ labels are in fact sufficient
(proof in Section \ref{sec:proof-loss}).

\vspace{-1.5mm}
\begin{theorem}\label{t:loss}
If the input matrix $\X\in\R^{d\times n}$ is in general position, 
then for any label vector $\y\in \R^n$, the expected
square loss (on all $n$ labeled vectors) of the optimal solution
$\of{\w^*}S$ for the subproblem 
$(\X_S,\y_S)$, with the $d$-element set $S$ obtained from 
volume sampling, is given by
\ifisarxiv\vspace{-1mm}\fi
\begin{align*}
\E[L(\of{\w^*}S)] =(d+1)\; L(\w^*).
\end{align*}
\ifisarxiv\vspace{-1mm}\fi
If $\X$ is not in general position, then the expected loss is
upper-bounded by $(d+1)\; L(\w^*)$.
\end{theorem}
\vspace{-0.5mm}
The factor $d+1$ cannot be improved when selecting only $d$
labels (we omit the proof): 
\begin{proposition}
\label{prop:optimal}
For any $d$, there exists a least squares problem $(\X,\y)$ with $d+1$
vectors in $\R^d$ such that for every $d$-element index set
$S\subseteq\{1,...,d+1\}$, we have \[L(\of{\w^*}S) = (d+1)\;L(\w^*).\]
\end{proposition}
\vspace{-.2cm}
Note that the multiplicative factor in Theorem \ref{t:loss} does not depend on
$n$. It is easy to see that this cannot be achieved by any
deterministic algorithm (without the access to labels). Namely,
suppose that $d=1$ and $\X$ is a vector of all ones, whereas the label
vector $\y$ is a vector of all ones except for a single zero. No
matter which column index we choose deterministically, if that index
corresponds to the label $0$, the solution to the subproblem will
incur loss $L(\of{\w^*}S)=n\, L(\w^*)$.
The fact that volume sampling is a joint distribution also plays an
essential role in proving Theorem \ref{t:loss}. Consider a matrix $\X$
with exactly $d$ unique linearly independent columns (and an arbitrary number of
duplicates). Any iid column sampling distribution (like for example
leverage score sampling) will require $\Omega(d\log d)$ samples to
retrieve all $d$ unique columns (ie coupon collector problem), which is
necessary to get any multiplicative loss bound. 

The exact expectation formula for the least squares loss under volume
sampling suggests a deep connection between linear regression and this
distribution. We can use Theorem \ref{t:einv} to further
strengthen that connection. Note, that the least squares estimator
obtained through volume sampling can be written as
$\of{\w^*}S=(\X\I_S)^{+\top}\y$.
Applying formula for the expectation of
pseudo-inverse, we conclude that $\of{\w^*}S$ is an unbiased estimator of
$\w^*$. 

\begin{proposition}\label{prop:unbiased}
Let $\X\in\R^{d\times n}$ be a full-rank matrix and $n \geq s\geq d$. Let
$S\subseteq 1..n$ be a size $s$ volume sampled set over $\X$. Then,
for arbitrary label vector $\y\in\R^n$, we have
\begin{align*}
\E[\of{\w^*}S] =\E[(\X\I_S)^{+\top}\y] = \X^{+\top}\y = \w^*. 
\end{align*}
\end{proposition}

For size $s=d$ volume sampling, the fact that $\E[\of{\w^*}S]$ equals $\w^*$
can be found in an old paper \cite{bental-teboulle}.
They give a direct proof based on Cramer's rule. 
For us the above proposition is a direct consequence of
the matrix expectation formula given in Theorem
\ref{t:einv} that holds for volume sampling of any size $s\ge d$. 
In contrast, the loss expectation formula of Theorem \ref{t:loss} is
limited to sampling of size $s=d$. Bounding the loss expectation for $s>d$ remains
an open problem. However, we consider a different strategy for extending volume
sampling in linear regression. Combining Proposition
\ref{prop:unbiased} with Theorem \ref{t:loss} we can compute the 
variance of predictions generated by volume sampling, and obtain
tighter multiplicative loss bounds by sampling multiple $d$-element
subsets $S_1,...,S_t$ independently. 
\begin{theorem}\label{t:repeated-sampling}
Let $(\X,\y)$ be as in Theorem \ref{t:loss}.  For $k$ independent
size $d$ volume samples $S_1,...,S_k$,
\[\E \left[L\left(
           \frac{1}{k}\sum_{j=1}^k\of{\w^*}{S_j}
      \right)\right]
= \left(1+\frac{d}{k}\right)\,L(\w^*).\]
\end{theorem}
\vspace{-.2cm}
\proof
Denote $\ybh\defeq\X^\top\w^*$ and $\yof{\ybh}S\defeq\X^\top\of{\w^*}S$
as the predictions generated by $\w^*$ and $\of{\w^*}S$ respectively. We  
perform bias-variance decomposition of the loss of $\of{\w^*}S$ (for
size $d$ volume sampling):
\begin{align*}\E[L(\of{\w^*}S)] &= \E[\|\yof{\ybh}S -
                             \y\|^2]=\E[\|\yof{\ybh}S - \ybh + \ybh - \y\|^2] \\
&=\E[\|\yof{\ybh}S - \ybh\|^2]
+ \E[2(\yof{\ybh}S-\ybh)^\top(\ybh-\y)] 
+ \|\ybh-\y\|^2\\
&\overset{(*)}{=}  \sum_{i=1}^n\E\left[(\yh{\ofsub{S}_i} - 
  \E[\yh\ofsub{S}_i])^2\right] + L(\w^*)=
\sum_{i=1}^n\Var[\yh{\ofsub{S}_i}] + L(\w^*),
\end{align*}
where $(*)$ follows from Theorem \ref{t:einv}. Now, we use
Theorem \ref{t:loss} to obtain the total variance of predictions:
\begin{align*}
\sum_{i=1}^n\Var[\yh{\ofsub{S}_i}] =\E[L(\of{\w^*}S)] - L(\w^*) =  d\;L(\w^*).
\end{align*}
Now the expected loss of the average weight vector
wrt sampling $k$ independent sets $S_1,...,S_k$ is: 
\begin{align*}
\hspace{0.5cm}\E \left[L\left(
           \frac{1}{k}\sum_{j=1}^k\of{\w^*}{S_j}
   \right)\right]
&= \sum_{i=1}^n\Var\left[\frac{1}{k}\sum_{j=1}^k\yh{\ofsub{S_j}_i}\right]
+L(\w^*)
 \\
&=  \frac{1}{k^2}\left(\sum_{j=1}^k d\,L(\w^*)\right) +
  L(\w^*)=\left(1 + \frac{d}{k}\right)L(\w^*). \hspace{2cm}\BlackBox
\end{align*}

It is worth noting that the average weight vector used in Theorem
\ref{t:repeated-sampling} is not expected to perform better than
taking the solution to the joint subproblem, $\of{\w^*}{S_{1:k}}$, where
$S_{1:k}= S_1\cup ...\cup S_k$. However, theoretical guarantees
for that case are not yet available.

\subsection{\bf Proof of Theorem \ref{t:loss}}
\label{sec:proof-loss}
We use the following lemma regarding the leave-one-out loss for
linear regression \cite{prediction-learning-games}:
\begin{lemma}\label{lm:leave-one-out}
Let $\of{\w^*}{-i}$ denote the least squares solution for problem
$(\X_{-i},\y_{-i})$. Then, we have
\begin{align*}
L(\w^*) =L(\of{\w^*}{-i}) - \x_i^\top(\X\X^\top)^{-1}\x_i \;
  \ell_i(\of{\w^*}{-i}), \quad\text{ where }\quad \ell_i(\w) \defeq (\x_i^\top\w - y_i)^2.
\end{align*}
\end{lemma}
When $\X$ has $d+1$ columns and $\X_{-i}$ is a
full-rank $d\times d$ matrix, then $L(\of{\w^*}{-i}) = \ell_i(\of{\w^*}{-i})$ and Lemma
\ref{lm:leave-one-out} leads to the following:
\vspace{-3mm}
\begin{align}
 \det(\Xs\Xs^\top) &\overset{(1)}{=} \det(\X\X^\top)\overbrace{\|\ybh - \y\|^2}^{L(\w^*)} 
\qquad \text{ where } \Xs=\left(\!\!\!\begin{array}{c}
\X \\
\y^\top  \end{array}\!\!\!\!\right) \nonumber\\
&\overset{(2)}{=} \det(\X\X^\top)
                     (1-\x_i^\top(\X\X^\top)^{-1}\x_i)
                     \ell_i(\of{\w^*}{-i}) \nonumber\\
&\overset{(3)}{=} \det(\X_{-i}\X_{-i}^\top) \ell_i(\of{\w^*}{-i}),\label{eq:simple-lemma}
\end{align}
where (1) is the ``base $\times$ height'' formula for volume, (2)
follows from Lemma \ref{lm:leave-one-out} and (3) follows from a
standard determinant formula.
Returning to the proof, our goal is to find the expected loss $\E[L(\of{\w^*}S)]$, where $S$
is a size $d$ volume sampled set.
First, we rewrite the expectation as follows:
\begin{align}
\E[L(\of{\w^*}S)] &= \sum_{S,|S|=d} P(S) L(\of{\w^*}S)
=\sum_{S,|S|=d} P(S) \sum_{j=1}^n \ell_j(\of{\w^*}S)\nonumber\\
&=\sum_{S,|S|=d}\sum_{j\notin S} P(S)\;\ell_j(\of{\w^*}S)
=\sum_{T,|T|=d+1}\sum_{j\in T}P(T_{-j})\;\ell_j(\of{\w^*}{T_{-j}}). \label{eq:sum-swap}
\end{align}
We now use (\ref{eq:simple-lemma}) on the matrix $\X_T$ and 
test instance $\x_j$ (assuming $\rank(\X_{T_{-j}})=d$):
\begin{align}
    \label{eq:summand}
P(T_{-j})\;\ell_j(\of{\w^*}{T_{-j}}) =
\frac{\det(\X_{T_{-j}}\X_{T_{-j}}^\top)}{\det(\X\X^\top)}\;\ell_j(\of{\w^*}{T_{-j}}) =
\frac{\det(\Xs_T \Xs_T^\top)}{\det(\X\X^\top)}.
\end{align}
Since the summand does not depend on the index $j\in T$,
the inner summation in (\ref{eq:sum-swap}) becomes a multiplication
by $d+1$.  This lets us write the expected loss as:  
\begin{align}
    \label{eq:th-cauchy-binet}
\E[L(\of{\w^*}S)] = \frac{d+1}{\det(\X\X^\top)}
\sum_{T,|T|=d+1}\!\!\det(\Xs_T \Xs_T^\top)  
\overset{(1)}{=} (d+1)\frac{\det(\Xs\Xs^\top)}{\det(\X\X^\top)}
\overset{(2)}{=} (d+1)\,L(\w^*), 
\end{align}
where (1) follows from the Cauchy-Binet formula
and (2) is an application of the ``base $\times$ height'' formula.
If $\X$ is not in general position, then for some summands in \eqref{eq:summand},
$\rank(\X_{T_{-j}})<d$ and $P(T_{-j})=0$.
Thus the left-hand side of \eqref{eq:summand} is $0$, while the right-hand
side is non-negative, so \eqref{eq:th-cauchy-binet} becomes an inequality,
completing the proof of Theorem \ref{t:loss}.



\section{Efficient algorithm for volume sampling}
\label{sec:algorithm}

In this section we propose an algorithm for efficiently performing
exact volume sampling for any $s\geq d$. This addresses the
question posed by \cite{avron-boutsidis13}, asking for a
polynomial-time algorithm for the case when
$s>d$. \cite{efficient-volume-sampling,more-efficient-volume-sampling}
gave an algorithm for the case when $s=d$, which runs in time
$O(nd^3)$. Recently, \cite{dual-volume-sampling} offered an algorithm
for arbitrary $s$, which has complexity $O(n^4 s)$. We propose a new method, which uses
our techniques to achieve the time complexity $O((n-s+d)nd)$, a direct
improvement over \cite{dual-volume-sampling} by a factor of at least
$n^2$. Our algorithm also offers an improvement for
$s=d$ in certain regimes. Namely, when $n=o(d^2)$, then our algorithm
runs in time $O(n^2d)=o(nd^3)$, faster than the method proposed by
\cite{efficient-volume-sampling}.

Our algorithm implements reverse iterative sampling from Corollary \ref{c:vol}. 
After removing $q$ columns, we are left with an index set
of size $n-q$ that is distributed according to volume sampling 
for column set size $n-q$.

\begin{theorem}
The sampling algorithm runs in time $O((n-s+d)nd)$,
using $O(d^2+n)$ additional memory,
and returns set $S$ which is distributed according to size
$s$ volume sampling over $\X$.
\end{theorem}
\begin{proof}
For correctness we show the following invariants
that hold at the beginning of the {\bf while} loop:
\begin{align*}
p_i = 1 - \x_i^\top(\X_S\X_S^\top)^{-1}\x_i = (|S|-d)\,P(\Sm|S)
\qquad
\text{and}
\qquad
\Z= (\X_S\X_S^\top)^{-1}.
\end{align*}
At the first iteration the invariants trivially hold. 
When updating the $p_j$ we use $\Z$ and the $p_i$ 
from the previous iteration, so we can rewrite the update as

\begin{wrapfigure}{R}{0.33\textwidth}
\vspace{-.6cm}
\renewcommand{\thealgorithm}{}
\begin{minipage}{0.33\textwidth}
\floatname{algorithm}{}
\begin{algorithm}[H] 
{\fontsize{8}{8}\selectfont
  \caption{\bf \small \hspace{-.2cm}Reverse iterative volume sampling}
  \begin{algorithmic}
    \STATE \textbf{Input:} $\X\!\in\!\R^{d\times n}$, $s\!\in\!\{d..n\}$ 
    \STATE $\Z\leftarrow (\X\X^\top)^{-1}$\label{line:inv}   
    \STATE $\forall_{i\in\{1..n\}} \quad p_i\leftarrow 1-\x_i^\top \Z\x_i$
    \STATE $S \leftarrow \{1,..,n\}$
    \STATE {\bf while} $|S|>s$
    \STATE \quad Sample $i \propto p_i$ out of $S$
    \STATE \quad $S\leftarrow S - \{i\}$
    \STATE \quad $\v \leftarrow \Z\x_i /\sqrt{p_i}$
    \STATE \quad $\forall_{j\in S}\quad  p_j\leftarrow p_j -  (\x_j^\top\v)^2$
    \STATE \quad $\Z \leftarrow \Z + \v\v^\top$  
    \STATE {\bf end} 
    \RETURN $S$
 \end{algorithmic}
\label{alg:sampling}
}
\end{algorithm}
\end{minipage}
\end{wrapfigure}
\begin{align*}
&p_j \leftarrow p_j - (\x_j^\top\v)^2 \\
&= 1- \x_j^\top(\X_S\X_S^\top)^{-1}\x_j -
      \frac{(\x_j^\top\Z\x_i)^2}{1-\x_i^\top(\X_S\X_S^\top)^{-1}\x_i}\\
&=1- \x_j^\top(\X_S\X_S^\top)^{-1}\x_j -
  \frac{\x_j^\top(\X_S\X_S^\top)^{-1}\x_i\x_i^\top(\X_S\X_S^\top)^{-1}\x_j}
{1-\x_i^\top(\X_S\X_S^\top)^{-1}\x_i}\\
&=1 - \x_j^\top\left( (\X_S\X_S^\top)^{-1} +
  \frac{(\X_S\X_S^\top)^{-1}\x_i\x_i^\top(\X_S\X_S^\top)^{-1}}
  {1-\x_i^\top(\X_S\X_S^\top)^{-1}\x_i}\right)\x_j  \\
&\overset{(*)}{=}
1- \x_j^\top(\X_\Sm\X_\Sm)^{-1}\x_j =(|S|-1-d)\,P(S_{-i,j}|\Sm),
\\[-.35cm]
\end{align*}
where $(*)$ follows from the Sherman-Morrison formula.
The update of $\Z$ is also an application of Sherman-Morrison 
and this concludes the proof of correctness. 

Runtime: Computing the
initial $\Z=(\X\X^\top)^{-1}$ takes $O(nd^2)$, as does
computing the initial values of $p_j$'s. Inside the \textbf{while}
loop, updating $p_j$'s takes $O(|S| d)=O(nd)$ and updating $\Z$ takes
$O(d^2)$. The overall runtime becomes $O(nd^2 + (n-s)nd) =
O((n-s+d)nd)$. The space usage (in addition to the input data) is
dominated by the $p_i$ values and matrix $\Z$.
\end{proof}

\section{Conclusions}
\label{sec:conclusions}
We developed exact formulas for 
$\E[(\X\I_S)^+)]$ and $\E[(\X\I_S)^+)^2]$
when the subset $S$ of $s$ column indices
is sampled proportionally to the volume $\det(\X_S\X_S^\top)$.
The formulas hold for any fixed size $s\in \{d..n\}$. 
These new expectation formulas imply that the solution $\of{\w^*}S$ 
for a volume sampled subproblem of a linear regression problem is
unbiased. We also gave a formula relating the loss of the subproblem 
to the optimal loss (ie $\E(L(\of{\w^*}S))=(d+1)L(\w^*)$). However, this
result only holds for sample size $s=d$. It is an open problem 
to obtain such an exact expectation formula for $s>d$.

A natural algorithm is to draw $k$ samples $S_i$ of size $d$
and return $\of{\w^*}{S_{1:k}}$, where $S_{1:k}=\bigcup_iS_i$.
We were able to get exact expressions for the
loss $L(\frac{1}{k} \sum_i \of{\w^*}{S_i})$ of the average
predictor but it is an open problem to get
nontrivial bounds for the loss of the best predictor $\of{\w^*}{S_{1:k}}$.

We were able to show that for small sample sizes, volume
sampling a set jointly has the advantage: It achieves
a multiplicative bound for the smallest sample size $d$,
whereas any independent sampling routine requires sample
size at least $\Omega(d \log d)$.

We believe that our results demonstrate a fundamental connection
between volume sampling and linear regression, which demands further
exploration. Our loss expectation formula has already been applied by
\cite{regression-correspondence} to the task of linear regression
without correspondence. 

\paragraph{Acknowledgements}
Thanks to Daniel Hsu and Wojciech Kot{\l}owski for many
valuable discussions.
This research was supported by NSF grant IIS-1619271.

\clearpage
\newpage
\bibliographystyle{plain}
\bibliography{pap}

\begin{thebibliography}{10}

\bibitem{avron-boutsidis13}
Haim Avron and Christos Boutsidis.
\newblock Faster subset selection for matrices and applications.
\newblock {\em SIAM Journal on Matrix Analysis and Applications},
  34(4):1464--1499, 2013.

\bibitem{bental-teboulle}
Aharon Ben-Tal and Marc Teboulle.
\newblock A geometric property of the least squares solution of linear
  equations.
\newblock {\em Linear Algebra and its Applications}, 139:165 -- 170, 1990.

\bibitem{coresets-regression}
Christos Boutsidis, Petros Drineas, and Malik Magdon{-}Ismail.
\newblock Rich coresets for constrained linear regression.
\newblock {\em CoRR}, abs/1202.3505, 2012.

\bibitem{prediction-learning-games}
Nicolo Cesa-Bianchi and Gabor Lugosi.
\newblock {\em Prediction, Learning, and Games}.
\newblock Cambridge University Press, New York, NY, USA, 2006.

\bibitem{regression-input-sparsity-time}
Kenneth~L. Clarkson and David~P. Woodruff.
\newblock Low rank approximation and regression in input sparsity time.
\newblock In {\em Proceedings of the Forty-fifth Annual ACM Symposium on Theory
  of Computing}, STOC '13, pages 81--90, New York, NY, USA, 2013. ACM.

\bibitem{efficient-volume-sampling}
Amit Deshpande and Luis Rademacher.
\newblock Efficient volume sampling for row/column subset selection.
\newblock In {\em Proceedings of the 2010 IEEE 51st Annual Symposium on
  Foundations of Computer Science}, FOCS '10, pages 329--338, Washington, DC,
  USA, 2010. IEEE Computer Society.

\bibitem{pca-volume-sampling}
Amit Deshpande, Luis Rademacher, Santosh Vempala, and Grant Wang.
\newblock Matrix approximation and projective clustering via volume sampling.
\newblock In {\em Proceedings of the Seventeenth Annual ACM-SIAM Symposium on
  Discrete Algorithm}, SODA '06, pages 1117--1126, Philadelphia, PA, USA, 2006.
  Society for Industrial and Applied Mathematics.

\bibitem{fast-leverage-scores}
Petros Drineas, Malik Magdon-Ismail, Michael~W. Mahoney, and David~P. Woodruff.
\newblock Fast approximation of matrix coherence and statistical leverage.
\newblock {\em J. Mach. Learn. Res.}, 13(1):3475--3506, December 2012.

\bibitem{optimal-design-book}
Valeri~Vadimovich Fedorov, W.J. Studden, and E.M. Klimko, editors.
\newblock {\em Theory of optimal experiments}.
\newblock Probability and mathematical statistics. Academic Press, New York,
  1972.

\bibitem{dpp-shopping}
Mike Gartrell, Ulrich Paquet, and Noam Koenigstein.
\newblock Bayesian low-rank determinantal point processes.
\newblock In {\em Proceedings of the 10th ACM Conference on Recommender
  Systems}, RecSys '16, pages 349--356, New York, NY, USA, 2016. ACM.

\bibitem{more-efficient-volume-sampling}
Venkatesan Guruswami and Ali~Kemal Sinop.
\newblock Optimal column-based low-rank matrix reconstruction.
\newblock In {\em Proceedings of the Twenty-third Annual ACM-SIAM Symposium on
  Discrete Algorithms}, SODA '12, pages 1207--1214, Philadelphia, PA, USA,
  2012. Society for Industrial and Applied Mathematics.

\bibitem{regression-correspondence}
Daniel Hsu, Kevin Shi, and Xiaorui Sun.
\newblock Linear regression without correspondence.
\newblock {\em CoRR}, abs/1705.07048, 2017.

\bibitem{dpp-clustering}
Byungkon Kang.
\newblock Fast determinantal point process sampling with application to
  clustering.
\newblock In {\em Proceedings of the 26th International Conference on Neural
  Information Processing Systems}, NIPS'13, pages 2319--2327, USA, 2013. Curran
  Associates Inc.

\bibitem{k-dpp}
Alex Kulesza and Ben Taskar.
\newblock {k-DPPs: Fixed-Size Determinantal Point Processes}.
\newblock In {\em {Proceedings of the 28th International Conference on Machine
  Learning}}, pages 1193--1200. {Omnipress}, 2011.

\bibitem{dpp}
Alex Kulesza and Ben Taskar.
\newblock {\em Determinantal Point Processes for Machine Learning}.
\newblock Now Publishers Inc., Hanover, MA, USA, 2012.

\bibitem{dual-volume-sampling}
C.~{Li}, S.~{Jegelka}, and S.~{Sra}.
\newblock {Column Subset Selection via Polynomial Time Dual Volume Sampling}.
\newblock {\em ArXiv e-prints}, March 2017.

\bibitem{randomized-matrix-algorithms}
Michael~W. Mahoney.
\newblock Randomized algorithms for matrices and data.
\newblock {\em Found. Trends Mach. Learn.}, 3(2):123--224, February 2011.

\bibitem{pool-based-active-learning-regression}
Masashi Sugiyama and Shinichi Nakajima.
\newblock Pool-based active learning in approximate linear regression.
\newblock {\em Mach. Learn.}, 75(3):249--274, June 2009.

\end{thebibliography}


\clearpage
\newpage

\end{document}